\documentclass{colt2014_arxiv}

\usepackage[disable]{todonotes}

\title[Online Linear Optimization via Smoothing]{Online Linear Optimization via Smoothing}
\usepackage{times}
\usepackage{pdfsync}

  \coltauthor{\Name{Jacob Abernethy} \Email{jabernet@umich.edu} \\  
   \Name{Chansoo Lee} \Email{chansool@umich.edu}\\
     \addr Computer Science and Engineering Division, University of Michigan, Ann Arbor
  \AND
   \Name{Abhinav Sinha} \Email{absi@umich.edu }\\
      \addr Electrical and Computer Engineering Division, University of Michigan, Ann Arbor
   \AND
    \Name{Ambuj Tewari} \Email{tewaria@umich.edu}\\
 \addr Department of Statistics, University of Michigan, Ann Arbor
   }

\usepackage{paralist}

\newcommand{\Regret}{\mathrm{Regret}}

\newcommand{\Tr}{\mathrm{Tr}}

\newcommand{\PP}{\mathbb{P}}
\newcommand{\ind}{\mathbf{1}}
\newcommand{\dom}{\mathrm{dom}}

\newcommand{\EE}{\mathbb{E}}
\newcommand{\RR}{\mathbb{R}}

\newcommand{\w}{w}
\newcommand{\f}{\Phi}
\newcommand{\tf}{\tilde{\f}}
\newcommand{\xx}{\theta}
\newcommand{\XX}{\Theta}

\newcommand{\cD}{\mathcal{D}}
\newcommand{\cX}{\mathcal{X}}
\newcommand{\cY}{\mathcal{Y}}
\newcommand{\cR}{\mathcal{R}}
\newcommand{\STDN}{\mathcal{N}(0, I)}

\newcommand{\X}{\mathbf{x}}

\newcommand{\ic}{\mathrm{ic}}

\DeclareMathOperator*{\argmax}{arg\,max}

\begin{document}
\maketitle


\begin{abstract}
We present a new optimization-theoretic approach to analyzing Follow-the-Leader style algorithms, particularly in the setting where perturbations are used as a tool for regularization. 
We show that adding a strongly convex penalty function to the decision rule and adding stochastic perturbations to data correspond to deterministic and stochastic smoothing operations, respectively. We establish an equivalence between ``Follow the Regularized Leader'' and ``Follow the Perturbed Leader'' up to the smoothness properties. This intuition leads to a new generic analysis framework that recovers and improves the previous known regret bounds of the class of algorithms commonly known as Follow the Perturbed Leader.
\end{abstract}

\section{Introduction}
In this paper, we study \emph{online learning} (other names include adversarial learning or no-regret learning) in which the learner iteratively plays actions based on the data received up to the previous iteration. The data sequence is chosen by an adversary and the learner's goal is to minimize the worst-case regret. The key to developing optimal algorithms is regularization, interpreted as \emph{hedging} against an adversarial future input and avoiding \emph{overfitting} to the observed data. In this paper, we focus on regularization techniques for online linear optimization problems where the learner's action is evaluated on a linear reward function.

Follow the Regularized Leader (FTRL) is an algorithm that uses explicit regularization \emph{via penalty function}, which directly changes the optimization objective.
At every iteration, FTRL selects an action by optimizing $\argmax_\w f(\w, \Theta) - \mathcal{R}(\w)$ where $f$ is the true objective, $\Theta$ is the observed data, and $\mathcal{R}$ is a strongly convex penalty function such as the well-known $\ell_2$-regularizer $\|\cdot\|_2$. 
The regret analysis of FTRL reduces to the analysis of the second-order behavior of the penalty function \citep{Sasha:2012:OLO}, which is well-studied due to the powerful convex analysis tools. 
In fact, regularization via penalty methods for online learning in general are very well understood. \cite{DBLP:conf/nips/SrebroST11} proved that Mirror Descent, a regularization via penalty method, achieves a nearly optimal regret guarantee for a general class of online learning problems, and \cite{DBLP:journals/jmlr/McMahan11a} showed that FTRL is equivalent to Mirror Descent under some assumptions.

Follow the Perturbed Leader (FTPL), on the other hand, uses implicit regularization \emph{via perturbations}.
At every iteration, FTPL selects an action by optimizing $\argmax_\w f(\w, \Theta + u)$ where $\XX$ is the observed data and $u$ is some random noise vector, often referred to as a ``perturbation'' of the input. Unfortunately, the analysis of FTPL lacks a generic framework and relies substantially on clever algebra tricks and heavy probabilistic analysis \citep{DBLP:journals/jcss/KalaiV05,Devroye:2013,Warmuth:COLT2014}. Convex analysis techniques, which led to our current thorough understanding of FTRL, have not been applied to FTPL, partly because the decision rule of FTPL does not explicitly contain a convex function.

In this paper, we present a new analysis framework that makes it possible to analyze FTPL in the same way that FTRL has been analyzed, particularly with regards to \emph{second-order properties} of convex functions.
We show that both FTPL and FTRL naturally arise as \emph{smoothing operations} of a non-smooth potential function and the regret analysis boils down to controlling the \emph{smoothing parameters} as defined in Section \ref{sec:smoothing}.
This new unified analysis framework not only recovers the known optimal regret bounds, but also gives a new type of generic regret bounds.

Prior to our work, \cite{DBLP:conf/nips/RakhlinSS12} showed that both FTPL and FTRL naturally arise as \emph{admissible relaxations} of the minimax value of the game between the learner and adversary. In short, adding a random perturbation and adding a regularization penalty function are both optimal ways to simulate the worst-case future input sequence.
We establish a stronger connection between FTRL and FTPL; both algorithms are derived from smoothing operations and they are \emph{equivalent} up to the smoothing parameters.
This equivalence is in fact a very strong result, considering the fact that \cite{harsanyi1973oddness} showed that there is no general bijection between FTPL and FTRL.

This paper also aligns itself with the previous work that studied the   connection between explicit regularization \emph{via penalty} and implicit regularization \emph{via perturbations}. \cite{Bishop:1995} showed that adding Gaussian noise to features of the training examples is equivalent to Tikhonov regularization, and more recently \cite{DBLP:conf/nips/WagerWL13} showed that for online learning, dropout training \citep{Hinton:dropout} is similar to AdagGrad \citep{duchi_adagrad} in that both methods scale features by the Fisher information. These results are derived from Taylor approximations, but our FTPL-FTRL connection is derived from the convex conjugate duality.

An interesting feature of our analysis framework is that we can directly apply existing techniques from the optimization literature, and conversely, our new findings in online linear optimization may apply to optimization theory. In Section \ref{sec:gaussian_general}, a straightforward application of the results on Gaussian smoothing by \cite{Nesterov11} and \cite{Duchi:SIAM2012} gives a generic regret bound for an arbitrary online linear optimization problem. In Section \ref{sec:gaussian_experts} and \ref{sec:gauss:l2l2}, we improve this bound for the special cases that correspond to canonical online linear optimization problems, and these results may be of interest to the optimization community.

\section{Preliminaries}
\subsection{Convex Analysis}
Let $f$ be a differentiable, closed, and proper convex function whose domain is $
 \dom f \subseteq \RR^N$. We say that $f$ is \emph{$L$-Lipschitz} with respect to a norm $\|\cdot\|$ when $f$ satisfies $|f(x) - f(y)| \leq L\|x-y\|$ for all $x,y\in\dom(f)$.

The \emph{Bregman divergence}, denoted $D_f(y,x)$, is the gap between $f(y)$ and the linear approximation of $f(y)$ around $x$. Formally, $D_f(y, x) = f(y) - f(x) - \langle \nabla f(x), y -x \rangle$. 
We say that $f$ is \emph{$\beta$-strongly convex} with respect to a norm $\|\cdot\|$ \, if we have $D_f(y,x) \geq \frac{\beta}{2} \|y - x\|^2$ for all $x,y \in \dom f$. Similarly, $f$ is said to be \emph{$\beta$-strongly smooth} with respect to a norm $\|\cdot\|$ if we have $ D_f(y,x) \leq \frac{\beta}{2} \|y - x\|^2$ for all $x,y \in \dom f$. The Bregman divergence measures how fast the gradient changes, or equivalently, how large the second derivative is. In fact, we can bound the Bregman divergence by analyzing the local behavior of Hessian, as the following adaptation of \citet[Lemma~4.6]{DBLP:journals/teco/AbernethyCV13} shows.
	\begin{lemma}
	\label{lem:jake_vHv}
	\renewcommand{\f}{f}
	\renewcommand{\xx}{x}
	 Let $\f$ be a twice-differentiable convex function with $
	 \dom f \subseteq \RR^N$. Let $\xx \in \dom f$, such that $v^T\nabla^2\f(\xx + \alpha v)v \in [a,b]$ ($a \leq b$) for all $\alpha \in [0,1]$.
	Then, $a\|v\|^2 /2 \leq D_{f}(\xx+v, \xx) \leq b\|v\|^2/2$.
	\end{lemma}

The \emph{Fenchel conjugate} of $f$ is $f^\star(\theta) = \sup_{w \in \dom(f)}\{ \langle w, \theta \rangle - f(w)\}$, and it is a dual mapping that satisfies $f = (f^\star)^\star$ and $\nabla f^\star \in \dom (f)$. By the strong convexity-strong smoothness duality, $f$ is $\beta$-strongly smooth with respect to a norm $\|\cdot\|$ if and only if $f^\star$ is $\frac{1}{\beta}$-strongly smooth with respect to the dual norm $\|\cdot\|_\star$. For more details and proofs, readers are referred to an excellent survey by \cite{Sasha:2012:OLO}.

\subsection{Online Linear Optimization}
 Let $\mathcal{X}$ and $\mathcal{Y}$ be convex and closed subsets of $\RR^N$. 
The online linear optimization is defined to be the following iterative process:

On round $t = 1,\ldots, T$,
\begin{compactitem}
	\item the learner plays $\w_t \in \mathcal{X}$.
	\item the adversary reveals $\xx_t \in \mathcal{Y}$.
	\item the learner receives a reward\footnote{Our somewhat less conventional choice of maximizing the reward instead of minimizing the loss was made so that we directly analyze the convex function $\max(\cdot)$ without cumbersome sign changes.} $\langle\w_t, \xx_t\rangle$.
\end{compactitem}
We say $\mathcal{X}$ is the \emph{decision set} and $\mathcal{Y}$ is the \emph{reward set}.
Let $\XX_t = \sum_{s=1}^{t} \xx_{s}$ be the cumulative reward. The learner's goal is to minimize the (external) regret, defined as:
\begin{equation}
\Regret = \underbrace{\max_{\w \in \mathcal{X}} \langle\w, \Theta_T \rangle}_{\text{baseline potential}} - \sum_{t=1}^{T} \langle w_t, \xx_t \rangle.
\end{equation}
The \emph{baseline potential function} $\Phi(\Theta) := \max_{\w \in \cX} \langle\w, \Theta \rangle$ is the comparator term against which we define the regret, and it coincides with the \emph{support function} of $\cX$.
For a bounded compact set $\cX$, the support function of $\cX$ is sublinear\footnote{A function $f$ is sublinear if it is positive homogeneous (i.e., $f(a x) =a f(x)$ for all $a > 0$) and subadditive (i.e., $f(x) + f(y) \geq f(x + y)$).}
and Lipschitz continuous with respect to any norm $\|\cdot\|$ with the Lipschitz constant $\sup_{x \in \cX}\|x\|$. 
For more details and proofs, readers are referred to \citet[Section~13]{rockafellar} or \citet[Appendix~F]{opac-b1102003}.

\section{Online Linear Optimization Algorithms via Smoothing}
\label{sec:smoothing}

\subsection{Gradient-Based Prediction Algorithm}
\label{sec:smoothing:gbpa}

Follow-the-Leader style algorithms solve an optimization objective every round and play an action of the form $w_t = \argmax_{w \in \cX} f(w, \XX_{t-1})$ given a fixed $\XX_{t-1}$. For example, Follow the Regularized Leader maximizes $f(w,\XX) = \langle w,\XX \rangle - \cR(w)$ where $\cR$ is a strongly convex regularizer, and Follow the Perturbed Leader maximizes $f = \langle w, \XX + u \rangle$ where $u$ is a random noise. A surprising fact about these algorithms is that there are many scenarios in which the action $w_t$ is exactly \emph{the gradient} of some scalar potential function $\Phi_t$ evaluated at $\XX_{t-1}$.
This perspective gives rise to what we call the Gradient-based Prediction Algorithm (GBPA), presented below. 
Note that \citet[Theorem~11.6]{DBLP:books/daglib/0016248} presented a similar algorithm, but our formulation eliminates all dual mappings. 
\todo[inline]{This algorithmic template fits into the broad category of Mirror Descent-style algorithms. (what ref? is it even true?)}

\begin{algorithm2e}
\caption{Gradient-Based Prediction Algorithm (GBPA)}
\label{alg:gftl}
\textbf{Input}: $\cX, \cY \subseteq \RR^N$\\
Initialize $\XX_0 =0$\\
\For{t = 1 to T}{
The learner chooses differentiable $\Phi_t: \RR^N \to \RR$ whose gradient satisfies $\text{Image}(\nabla\Phi_t) \subseteq \cX$ \\
The learner plays $w_{t} = \nabla \Phi_{t}(\XX_{t-1})$ \\
The adversary reveals $\theta_{t} \in \mathcal{Y}$ and the learner gets a reward of $\langle w_t, \theta_t\rangle$ \\
Update $\XX_{t} = \XX_{t-1} + \theta_{t}$
}
\end{algorithm2e}

\vspace{-10pt}
\begin{lemma}[GBPA Regret]
\label{lem:genericregret1}
Let $\f$ be the baseline potential function for an online linear optimization problem. The regret of the GBPA can be written as:
\begin{align}
\label{eq:gbparegret}
\Regret &= 
\underbrace{\f(\XX_T) - \f_T(\XX_T)}_{\text{underestimation penalty}} 
+ \sum_{t = 1}^{T} \bigg(\underbrace{(\f_t(\XX_{t-1}) - \f_{t-1}(\XX_{t-1}))}_{\text{overestimation penalty}} 
+ \underbrace{D_{\f_{t}}(\XX_t, \XX_{t-1})}_{\text{divergence penalty}} \bigg),
\end{align}
where $\f_0 \equiv \f$.
\end{lemma}
\begin{proof}
	See Appendix \ref{app:gbpa_regret}.
\end{proof}

In the existing FTPL analysis, the counterpart of the divergence penalty  is $\langle \w_{t + 1} - \w_t, \xx_{t}\rangle$, which is controlled by analyzing the probability that the noise would cause the two random variables $\w_{t+1}$ and $\w_t$ to differ. 
In our framework, $\w_{t}$ is the gradient of a function $\f_t$ of $\XX$, which means that if $\f_t$ is twice-differentiable, we can take the derivative of $\w_t$ with respect to $\XX$. This derivative is the Hessian matrix of $\f_t$, which essentially controls $\langle \w_t - \w_{t-1}\rangle$ with the help of Lemma~\ref{lem:jake_vHv}. Since we focus on the curvature property of \emph{functions} as opposed to random vectors, our FTPL analysis involves less probabilistic analysis than \cite{Devroye:2013} or \cite{Warmuth:COLT2014} does.

We point out a couple of important facts about Lemma \ref{lem:genericregret1}: 
\begin{compactenum}[(a)]
	\item If $\f_1 \equiv \cdots \equiv \f_T$, then the overestimation penalty sums up to $\f_{1}(0) - \f(0) = \f_{T}(0) - \f(0)$.
	\item If $\f_t$ is $\beta$-strongly smooth with respect to $\|\cdot\|$, the divergence penalty at $t$ is at most $\frac\beta 2 \|\theta_t\|^2$.
\end{compactenum}

\todo[inline]{Jake: if you have time, check the bridging paragraph above}
\subsection{Smoothability of the Baseline Potential}
\label{sec:smoothing:generic}
Equation \ref{eq:gbparegret} shows that the regret of the GBPA can be broken into two parts. One source of regret is the Bregman divergence of $\f_t$; since $\xx_t$ is not known until playing $\w_t$, the GBPA always ascends along the gradient that is one step behind. The adversary can exploit this and play $\xx_t$ to induce a large \emph{gap} between $\f_t(\X_t)$ and the linear approximation of $\f_t(\XX_t)$ around $\XX_{t-1}$.
Of course, the learner can reduce this gap by choosing a \emph{smooth} $\f_t$ whose gradient changes slowly. The learner, however, cannot achieve low regret by choosing an arbitrarily smooth $\f_t$, because the other source of regret is the difference between $\f_t$ and $\f$. In short, the GBPA achieves low regret if the potential function $\f_t$ gives a favorable tradeoff between the two sources of regret. This tradeoff is captured by the following definition of \emph{smoothability}.

\begin{definition}
\label{def:smoothable}
\cite[Definition 2.1]{journals/siamjo/BeckT12} Let $\f$ be a closed proper convex function. A collection of functions $\{\hat\f_\eta: \eta \in \mathbb{R}\}$ is said to be an \emph{$\eta$-smoothing} of a \emph{smoothable function} $\f$ with \emph{smoothing parameters} $(\alpha, \beta, \|\cdot\|)$, if for every $\eta >0$
	\begin{compactenum}[(i)]
		\item There exists $\alpha_1$ (underestimation bound) and $\alpha_2$ (overestimation bound) such that 
		\[\sup_{\XX \in \dom(\f)}\f(\XX) - \hat\f_\eta(\XX) \leq \alpha_1\eta \text{\,\, and\,\, } \sup_{\XX \in \dom(\f)} \hat\f_\eta(\XX) -\f(\XX) \leq \alpha_2\eta\]
		with $\alpha_1+\alpha_2 = \alpha$.
		\item $\hat\f_\eta$ is $\frac{\beta}{\eta}$-strongly smooth with respect to $\|\cdot\|$.
	\end{compactenum}
We say $\alpha$ is the \emph{deviation parameter}, and $\beta$ is the \emph{smoothness parameter}.
\end{definition}
A straightforward application of Lemma \ref{lem:genericregret1} gives the following statement:
\begin{corollary} 
\label{lem:genericregret2}
Let $\f$ be the baseline potential for an online linear optimization problem. Suppose $\{\hat\f_\eta\}$ is an $\eta$-smoothing of $\f$ with parameters $(\alpha, \beta, \|\cdot\|)$.
Then, the GBPA run with $\f_1 \equiv \cdots \equiv \f_T \equiv \hat\f_\eta$ has regret at most \[\Regret \leq \alpha\eta + \frac{\beta}{2\eta}\sum_{t=1}^{T}\|\xx_t\|^2\]
\end{corollary}
In online linear optimization, we often consider the settings where the marginal reward vectors $\theta_1, \ldots, \theta_t$ are constrained by a norm, i.e., $\|\theta_t\| \leq r$ for all $t$. In such settings, the regret grows in $O(\sqrt{r\alpha\beta T})$ for the optimal choice of $\alpha$. The product $\alpha\beta$, therefore, is at the core of the GBPA regret analysis.

\subsection{Algorithms}
\label{sec:smoothing:algo}
\paragraph{Follow the Leader (FTL)} Consider the GBPA run with a fixed potential function $\f_t \equiv \f$ for $t=1, \ldots, T$, i.e., the learner chooses the baseline potential function every iteration. At iteration $t$, this algorithm plays $\nabla \f_t(\XX_{t-1}) = \arg\max_{\w} \langle \w, \XX_{t-1} \rangle$, which is equivalent to FTL \cite[Section~3.2]{DBLP:books/daglib/0016248}. FTL suffers zero regret from the over- or underestimation penalty, but the divergence penalty grows linearly in $T$ in the worst case, resulting in an $\Omega(T)$ regret.


\paragraph{Follow the Regularized Leader (FTRL)} Consider the GBPA run with a \emph{regularized potential}:
\begin{equation}
	\label{eq:ftrl_potential}
	\forall t, \f_t(\XX) = \mathcal{R}^\star(\XX) = \max_{\w \in \mathcal{X}}\{ \langle \w, \XX \rangle - \mathcal{R}(\w)\} 
\end{equation}
 where $\mathcal{R}: \mathcal{X} \to \RR$ is a $\beta$-strongly convex function. At time $t$, this algorithm plays $\nabla\f_t(\XX_{t-1}) = \arg\max_{\w}\{\langle \w, \XX_{t-1} \rangle - \mathcal{R}(\w)\}$, which is equivalent to FTRL. By the strong convexity-strong smoothness duality, $\Phi_t$ is $\frac{1}{\beta}$-strongly smooth with respect to the dual norm $\|\cdot\|_\star$. In Section \ref{sec:infconv}, we give an alternative interpretation of FTRL as a deterministic smoothing technique called inf-conv smoothing.

\paragraph{Follow the Perturbed Leader (FTPL)} Consider the GBPA run with a \emph{stochastically smoothed potential}:
\begin{equation}
	\label{eq:ftpl_potential}
	\forall t, \f_t(\XX) = \tf(\XX; \eta, \mathcal{D}) \overset{\text{def}}{=} \EE_{u \sim \mathcal{D}}[\f(\XX +\eta u)] = \EE_{u \sim \mathcal{D}}\Big[ \max_{\w \in \mathcal{X}}\{ \langle \w, \XX + \eta u \rangle\}\Big] 
\end{equation}
where $\mathcal{D}$ is a smoothing distribution with the support $\RR^N$ and $\eta >0$ is a \emph{scaling parameter}. This technique of \emph{stochastic smoothing} has been well-studied in the optimization literature for gradient-free optimization algorithms \citep{glasserman1991gradient,Yousefian:2010} and accelerated gradient methods for non-smooth optimizations \citep{Duchi:SIAM2012}. If the $\max$ expression inside the expectation has a unique maximizer with probability one, we can swap the expectation and gradient \citep[Proposition~2.2]{bertsekas1973} to obtain
\begin{equation}
\label{eq:gbpa-ftpl}
	\nabla\f_t(\XX_{t-1}) = \EE_{u \sim \mathcal{D}}\Big[\argmax_{\w \in \mathcal{X}}\{ \langle \w, \XX_{t-1} + \eta u \rangle\}\Big].
\end{equation}
Each $\argmax$ expression is equivalent to the decision rule of FTPL \citep{hannan1957, DBLP:journals/jcss/KalaiV05}; the GBPA on a stochastically smoothed potential can thus be seen as playing the \emph{expected action} of FTPL. Since the learner gets a linear reward in online linear optimization, the regret  of the GBPA on a stochastically smoothed potential is equal to the \emph{expected regret} of FTPL.

\paragraph{FTPL-FTRL Duality}
Our potential-based formulation of FTRL and FTPL reveals that a strongly convex regularizer defines a smooth potential function via duality, while adding perturbations is a \emph{direct} smoothing operation on the baseline potential function. By the strong convexity-strong smoothness duality, if the stochastically smoothed potential function is $(1/\beta)$-strongly smooth with respect to $\|\cdot\|_\star$, then its Fenchel conjugate implicitly defines a regularizer that is $\beta$-strongly convex with respect to $\|\cdot\|$.

This connection via duality is a \emph{bijection} in the special case where the decision set is one-dimensional. Previously it had been observed\footnote{Adam Kalai first described this result in personal communication and \cite{wm_gumbel} expanded it into a short note available online. However, the result appears to be folklore in the area of probabilistic choice models, and it is mentioned briefly in \citet{sandholm}.} that the Hedge Algorithm \citep{Freund1997119}, which can be cast as FTRL with an entropic regularization $\cR(\w) = \sum_i \w_i \log \w_i$, is equivalent to FTPL with Gumbel-distributed noise. 
\citet[Section~2]{sandholm} gave a generalization of this fact to a much larger class of perturbations, although they focused on repeated game playing where the learner's decision set $\mathcal{X}$ is the probability simplex. The inverse mapping from FTPL to FTRL, however, does not appear to have been previously published.
\begin{theorem}
	\label{thm:ftpl-ftrl-bijecction}
	Consider the one-dimensional online linear optimization problem with $\cX = \cY = [0,1]$. 
	Let $\mathcal{R}: \mathcal{X} \to \RR$ be a strongly convex regularizer. Its Fenchel conjugate $\mathcal{R}^\star$ defines a valid CDF of a continuous distribution $\mathcal{D}$ such that Equation \ref{eq:ftrl_potential} and Equation \ref{eq:ftpl_potential} are equal. 
	Conversely, let $F_\mathcal{D}$ be a CDF of a continuous distribution $\cD$ with a finite expectation. If we define $\mathcal{R}$ to be such that $\cR(w) - \cR(0) = - \int_0^w F_{\mathcal{D}}^{-1}(1-z) \mathrm{d}z$, then Equation \ref{eq:ftrl_potential} and Equation \ref{eq:ftpl_potential} are equal.
\end{theorem}
\begin{proof} In Appendix \ref{app:ftpl-ftrl-bijecction}.
	\end{proof}
\todo[inline]{To Jake: Please check the theorem (esp. conditions on the distribution and regularizer) and proof carefully}
\section{Online Linear Optimization via Gaussian Smoothing}

\label{sec:gaussian}


Gaussian smoothing is a standard technique for smoothing a function. In computer vision applications, for example, image pixels are viewed as a function of the $(x,y)$-coordinates, and Gaussian smoothing is used to blur noises in the image. We first present basic results on Gaussian smoothing from the optimization literature.
\todo[inline]{Jake: maybe add a sentence on why Gaussian is the standard kernel for convolving a function?}

\begin{definition}[Gaussian smoothing] Let $\f: \RR^N \to \RR$ be a function. Then, we define its Gaussian smoothing, with a scaling parameter $\eta > 0$ and a covariance matrix $\Sigma$, as
\[\tilde{\f}(\XX; \eta, \mathcal{N}(0, \Sigma)) = \EE_{u \sim \mathcal{N}(0, \Sigma)} \f(\XX 
+ \eta u) = (2\pi)^{-\frac{N}{2}} \mathrm{det}(\Sigma)^{-\frac{1}{2}}\int_{\RR^N} \f(\XX + \eta u) e^{-\frac{1}{2} u^T \Sigma^{-1} u} \ du\]
\end{definition}
In this section, when the smoothing parameters are clear from the context, we use a shorthand notation $\tf$. An extremely useful property of Gaussian smoothing is that $\tilde{\f}$ is always twice-differentiable, even when $\f$ is not. The trick is to introduce a new variable $\tilde{\XX} = \XX + \eta u$. After substitutions, the variable $\XX$ only appears in the exponent, which can be safely differentiated.

\begin{lemma}
\label{lem:smoothinggd}
(\citealt[Lemma~2]{Nesterov11}, and \citealt[Section~3]{Bhatnagar07adaptivenewton-based})
Let $\f:\RR^N \to \RR$ be a function. For any positive $\eta$, $\tilde{\f}(\cdot \ ; \eta, \mathcal{N}(0, \Sigma))$ is twice-differentiable and
	\begin{equation}
	\label{eq:gaussian_gradient}
		\nabla \tilde{\f}(\XX; \eta, \mathcal{N}(0, \Sigma)) = \frac{1}{\eta}\EE_u [\f(\XX + \eta u) \Sigma^{-1} u]
	\end{equation}
	\begin{equation}
	\label{eq:gaussian_hessian}
		\nabla^2 \tilde{\f}(\XX; \eta, \mathcal{N}(0, \Sigma)) = \frac{1}{\eta^2}\EE_u \Big[\f(\XX + \eta u) \Big((\Sigma^{-1}u) (\Sigma^{-1} u)^T - \Sigma^{-1} \Big) \Big]	
	\end{equation}
\end{lemma}

If $\f(\XX + \eta u)$ is differentiable almost everywhere, then we can directly differentiate Equation \ref{eq:gaussian_gradient} by swapping the expectation and gradient \citep[Proposition~2.2]{bertsekas1973} and obtain an alternative expression for Hessian:
	\begin{equation}
	\label{eq:gaussian_hessian2}
		\nabla^2 \tilde{\f}(\XX; \eta, \mathcal{N}(0, \Sigma)) = \frac{1}{\eta}\EE_u [\nabla \f\big(\XX + \eta u) (\Sigma^{-1}u)^T].
	\end{equation}
\todo[inline]{I wanted to say: Because the Hessian of a potential function is always symmetric positive semidefinite matrix, this alternative Hessian expression reveals a very interesting symmetry. But the Hessian is symmetric if potential has continuous 2nd derivative everywhere. is that guaranteed to be true?}



\subsection{Experts Setting \texorpdfstring{($\ell_1$-$\ell_\infty$ case)}{l1}}
\label{sec:gaussian_experts}
The experts setting is where $\mathcal{X} = \Delta^N \overset{\text{def}}{=} \{\w \in \RR^N: \sum_{i} \w_i = 1, \w_i \geq 0 \ \forall i\}$, and $\mathcal{Y} = \{\xx \in \RR^N: \|\xx\|_\infty \leq 1\}$. The baseline potential function is $\f(\XX) = \max_{\w \in \mathcal{X}}\langle \w, \XX\rangle = \XX_{i^*(\XX)},$ where we define $\textstyle i^*(z) := \min \{i: i \in \arg\max_{j} z_{j}\}$.

Our regret bound in Theorem \ref{thm:experts_regret} is data-dependent, and it is stronger than the previously known $O(\sqrt{T \log N})$ regret bounds of the algorithms that use similar perturbations. In the game theoretic analysis of Gaussian perturbations by \cite{DBLP:conf/nips/RakhlinSS12}, the algorithm uses the scaling parameter $\eta_t = \sqrt{T - t}$, which  requires the knowledge of $T$ and does not adapt to data. \cite{Devroye:2013} proposed the Prediction by Random Walk (PRW) algorithm, which flips a fair coin every round and decides whether to add $1$ to each coordinate. Due to the discrete nature of the algorithm, the analysis must assume the worst case where $\|\theta_t\|_\star = 1$ for all $t$.

\begin{theorem}
\label{thm:experts_regret}
	Let $\f$ be the baseline potential for the experts setting. The GBPA run with the Gaussian smoothing of $\f$, i.e., $\f_t(\cdot) = \tilde{\f}(\cdot; \eta_t, \mathcal{N}(0, I))$ for all $t$ has  regret at most
	\begin{equation}
		\label{eq:gauss_experts_regret}
		\Regret \leq \sqrt{2\log N} {\textstyle \left(\eta_T + \sum_{t=1}^{T} \frac{1}{\eta_t} \|\theta_t\|_{\infty}^{2} \right)}.
	\end{equation}
		If the algorithm selects $\eta_t = \sqrt{\sum_{t=1}^{T} \|\theta_t\|_{\infty}^{2}}$ for all $t$ (with the help of hindsight), we have
	\[
		\Regret \leq 2 \sqrt{\textstyle 2\sum_{t=1}^{T} \|\theta_t\|_{\infty}^{2} \log N }.
	\]
	If the algorithm selects $\eta_t$ adaptively according to $\eta_t = \sqrt{2(1 + \sum_{s=1}^{t-1} \|\theta_s\|_{\infty}^{2})}$, we have
	\[
		\Regret \leq 4\sqrt{\textstyle (1 + \sum_{t=1}^{T} \|\theta_t\|_{\infty}^{2}) \log N}.
	\]
\end{theorem}

\begin{proof}
	In order to apply Lemma \ref{lem:genericregret1}, we need to upper bound (i) the overestimation and underestimation penalty, and (ii) the Bregman divergence. To bound (i), first note that due to convexity of $\f$, the smoothed potential $\tf$ is also convex and upper bounds the baseline potential. 
	Hence, the underestimation penalty is at most 0, and when $\eta_t$ is fixed for all $t$, it is straightforward to bound the overestimation penalty:
		\begin{equation}
		\label{eq:maxgauss}
			\f_T(0) - \f(0) \leq \EE_{u \sim \mathcal{N}(0, I)}[\f(\eta_T u)] \leq \eta_T \sqrt{2\log N}.
		\end{equation}
	The first inequality is the triangle inequality. The second inequality is a well-known result and we included the proof in Appendix \ref{app:maxgauss} for completeness. For the adaptive $\eta_t$, we apply Lemma \ref{lem:changingf}, which we prove at the end of this section, to get the same bound.

	It now remains to bound the Bregman divergence. This is achieved in Lemma~\ref{lem:experts_hessian} where we upper bound $\sum_{i, j} |\nabla^2_{ij}\f|$, which is an upper bound on $\max_{\xx: \|\xx\|_\infty = 1} \xx^T (\nabla^2\f) \xx$. The final step is to apply Lemma \ref{lem:jake_vHv}. 
\end{proof}

The proof of Theorem \ref{thm:experts_regret} shows that for the experts setting, the Gaussian smoothing is an $\eta$-smoothing with parameters $(O(\sqrt{\log N}), O(\sqrt{\log N}), \|\cdot\|)$. This is in contrast to the Hedge Algorithm \citep{Freund1997119}, which is an $\eta$-smoothing with parameters $(\log N, 1, \|\cdot\|)$ (See Section \ref{sec:infconv} for details). Interestingly, the two algorithms obtain the same optimal regret (up to constant factors) although they have different smoothing parameters.

\todo[inline]{check the above paragraph}
\todo[inline,color=black!20]{Any implication? interpretation? Why one would prefer one kind of tradeoff over the other?}

	\begin{lemma}
	\label{lem:experts_hessian}
	Let $\Phi$ be the baseline potential for the experts setting. Let the Hessian matrix of the Gaussian-smoothed baseline potential be denoted $H$, i.e., $H = \nabla^2\tilde{\f} (\XX; \eta, \mathcal{N}(0,I))$. Then,
	\[\sum_{i,j} |H_{ij}| \leq \frac{2\sqrt{2 \log N}}{\eta}.\]
	\end{lemma}

	\begin{proof}
	With probability one, $\f(\Theta + \eta u)$ is differentiable and from Lemma \ref{lem:smoothinggd}, we can write
	\[H = \frac{1}{\eta}\EE[ \nabla \f(\XX + \eta u) u^T] = \EE[e_{i^*(\eta u + \XX)}u^T],\] where $e_i \in \RR^N$ is the $i$-th standard basis vector.

	First, we note that all off-diagonals of $H$ are negative and all diagonal entries in $H$ are positive. 
	This is because the Hessian matrix is the covariance matrix between the probability that $i$-th coordinate is the maximum and the extra random Gaussian noise added to the $j$-th coordinate; 
	for any positive number $\alpha$, $u_j = \alpha$ and $u_j = -\alpha$ have the same probability, but the indicator for $i = i^*$ has a higher probability to be 1 when $u_i$ is positive (hence $H_{ii} >0$) and $u_j$ is negative for $i \neq j$ (hence $H_{ij} <0$ for $i \neq j$).

	Second, the entries of $H$ sum up to 0, as
	\[\sum_{i,j} H_{ij} = \frac{1}{\eta}\EE\left[\textstyle \sum_j u_j \sum_i \ind\{i = i^*(\XX + u) \} \right] = \frac{1}{\eta}\EE\left[\textstyle \sum_j u_j\right] = 0.\]

	Combining the two observations, we have
	\[\sum_{i,j} |H_{ij}| = \sum_{i,j:H_{ij} > 0} H_{ij} +\sum_{i,j:H_{ij} < 0} -H_{ij}  =  2\sum_{i,j:H_{ij} > 0} H_{ij} = 2 \Tr(H)\].

	Finally, the trace is bounded as follows:
	\begin{align*}
	\Tr(H) =  \frac{1}{\eta}\EE\Big[\sum_i u_i \ind\{ i = i^*(\XX + u)\}\Big] 
	&\leq\frac{1}{\eta} \EE\Big[(\max_k u_k) \sum_i \ind\{i = i^*(\XX + u)\}\Big] \\
	&= \frac{1}{\eta}\EE[\max_k u_k] \leq \frac{1}{\eta}\sqrt{2\log N},
	\end{align*}
	where the final inequality is shown in Appendix~\ref{app:maxgauss}. Multiplying both sides by 2 completes the proof.
	\end{proof}

\paragraph{Time-Varying Scaling Parameters} When the scaling parameter $\eta_t$ changes every iteration, the overestimation penalty becomes a sum of $T$ terms. The following lemma shows that using the sublinearity of the baseline potential, we can collapse them into one.

\begin{lemma}
\label{lem:changingf}
Let $\f:\RR^N \to \RR$ be a sublinear function, and $\mathcal{D}$ be a continuous distribution with the support $\RR^N$.
Let $\f_t(\Theta) = \tilde{\f}(\Theta; \eta_t, \cD)$ for $t=0, \ldots, T$ and choose $\eta_t$ to be a non-decreasing sequence of non-negative numbers ($\eta_0 = 0$ so that $\f_0 = \f$). Then, the overestimation penalty in Equation \ref{eq:gbparegret} has the following upper bound:
\[
\sum_{t=1}^{T} \f_t(\Theta_{t-1}) - \f_{t-1}(\Theta_{t-1}) \leq \eta_T \EE_{u \sim \mathcal{D}}[\f(u)].
\]
\end{lemma}
\begin{proof} See Appendix \ref{app:changingf}
\end{proof}

\subsection{Online Linear Optimization over Euclidean Balls \texorpdfstring{($\ell_2$-$\ell_2$ case)}{l2l2}}
\label{sec:gauss:l2l2}

The Euclidean balls setting is where $\mathcal{X} = \mathcal{Y} = \{x \in \RR^N: \|x\|_2 \leq 1\}$. The baseline potential function is $\f(\XX) = \max_{\w \in \cX}\langle \w, \XX \rangle =  \|\XX\|_2$. We show that the GBPA with Gaussian smoothing achieves a minimax optimal regret \citep{DBLP:conf/colt/AbernethyBRT08} up to a constant factor.

\begin{theorem}
\label{thm:l2l2_regret}
Let $\f$ be the baseline potential for the Euclidean balls setting. The GBPA run with $\f_t(\cdot) = \tilde{\f}( \cdot ; \eta, \mathcal{N}(0, I))$ for all $t$ has regret at most
\begin{equation} \textstyle
\label{eq:gauss_l2l2_regret}
\Regret \leq \eta_T \sqrt{N} + \frac{1}{2\sqrt{N}}\sum_{t=1}^{T} \frac{1}{\eta_t} \|\xx_t\|_{2}^{2}.
\end{equation}
		If the algorithm selects $\eta_t = \sqrt{\sum_{s=1}^{T} \|\xx_s\|_{2}^{2} / (2N)}$ for all $t$ (with the help of hindsight), we have
	\[
		\Regret \leq \sqrt{\textstyle 2\sum_{t=1}^{T} \|\xx_t\|_{2}^{2}}.
	\]
	If the algorithm selects $\eta_t$ adaptively according to $\eta_t = \sqrt{(1 + \sum_{s=1}^{t-1}\|\xx_s\|_{2}^{2})) / N}$, we have
	\[
		\Regret \leq 2\sqrt{\textstyle 1 + \sum_{t=1}^{T} \|\xx_t\|_{2}^{2}}
	\]
\end{theorem}

\begin{proof} The proof is mostly similar to that of Theorem \ref{thm:experts_regret}. 	In order to apply Lemma \ref{lem:genericregret1}, we need to upper bound (i) the overestimation and underestimation penalty, and (ii) the Bregman divergence.

The Gaussian smoothing always overestimates a convex function, so it suffices to bound the overestimation penalty. Furthermore, it suffices to consider the fixed $\eta_t$ case due to Lemma \ref{lem:jake_vHv}.
The overestimation penalty can be upper-bounded as follows:
\begin{align*}
\f_T(0) - \f(0) = \EE_{u \sim \STDN} \|\XX+\eta_T u\|_2 - \|\XX\|_2 &\leq \eta_T\EE_{u \sim \STDN}\|u\|_2  \\
&\leq \eta_T \sqrt{\EE_{u \sim \STDN}\|u\|_2^2} = \eta_T\sqrt{N}	
\end{align*}
The first inequality is from the triangle inequality, and the second inequality is from the concavity of the square root. 

For the divergence penalty, note that the upper bound on $\max_{v: \|\xx\|_2 = 1} \xx^T (\nabla^2\tf) \xx$ is exactly the maximum eigenvalue of the Hessian, which we bound in Lemma \ref{thm:hessian_l2}. The final step is to apply Lemma \ref{lem:jake_vHv}.
\end{proof}

\begin{lemma}
\label{thm:hessian_l2} Let $\f$ be the baseline potential for the Euclidean balls setting. Then, for all $\Theta \in \RR^N$ and $\eta >0$, the Hessian matrix of the Gaussian smoothed potential satisfies
\[\textstyle \nabla^2 \tilde{\f} (\XX; \eta, \mathcal{N}(0,I)) \preceq \frac{1}{\eta\sqrt{N}}I. \]
\end{lemma}

\begin{proof}
The Hessian of the Euclidean norm 
$\nabla^2\Phi(\XX) = \|\XX\|_2^{-1}I - \|\XX\|_2^{-3}\XX \XX^T$
diverges near $\XX = 0$. Expectedly, the maximum curvature is at origin even after Gaussian smoothing (See Appendix \ref{app:gaussian-l2worst}). So, it suffices to prove
\[\textstyle \nabla^2 \f(0) =\EE_{u \sim \mathcal{N}(0, I)}[\|u\|_2(uu^T - I)] \preceq \sqrt{\frac{1}{N}}I,\]
where the Hessian expression is from Equation \ref{eq:gaussian_hessian2}.
\todo[inline,color=black!20]{Chansoo to Everyone: Shouldn't this actually be obvious? the worst curvature point should remain the same because Gaussian smoothing will always average out the curvature using a symmetric distribution.}

 By symmetry, all off-diagonal elements of the Hessian are 0. Let $Y = \|u\|^2$, which is Chi-squared with $N$ degrees of freedom. So, 
\[
	\Tr( \EE[ \|u\|_2 (uu^T - I) ] ) = \EE[ \Tr(\|u\|_2 (uu^T - I)) ] = \EE[ \|u\|_2^3 - N\|u\|_2 ] = \EE[Y^{\frac{3}{2}}] - N \EE[Y^{\frac{1}{2}}]
\]
Using the Chi-Squared moment formula \cite[p. 20]{Harvey1965}, the above becomes:
\begin{equation}
\label{eq:chisq}
	\frac{2^{\frac{3}{2}} \Gamma(\frac{3}{2} + \frac{N}{2})}{\Gamma(\frac{N}{2})} - \frac{N 2^{\frac{1}{2}}\Gamma(\frac{1}{2} + \frac{N}{2})}{\Gamma(\frac{N}{2})} = \frac{\sqrt{2} \Gamma(\frac{1}{2} + \frac{N}{2})}{\Gamma(\frac{N}{2})}.
\end{equation}
From the log-convexity of the Gamma function,
\[\textstyle \log \Gamma\left(\frac{1}{2} + \frac{N}{2}\right) \leq
 \frac{1}{2}\left(\log \Gamma \left(\frac{N}{2}\right) + \log \Gamma\left(\frac{N}{2} + 1\right) \right) 
 = \log \Gamma\left(\frac{N}{2}\right)\sqrt{\frac{N}{2}} .\]
Exponentiating both sides, we obtain
\[
	\textstyle \Gamma\left(\frac{1}{2} + \frac{N}{2}\right) \leq \Gamma\left(\frac{N}{2}\right) \sqrt{\frac{N}{2}},
\]
which we apply to Equation \ref{eq:chisq} and get $\Tr (\nabla^2 \f(0)) \leq \sqrt{N}$. To complete the proof, note that by symmetry, each entry must have the same expected value, and hence it is bounded by $\sqrt{1 / N}$.
\end{proof}

\subsection{General Bound}
\label{sec:gaussian_general}
In this section, we will use a generic property of Gaussian smoothing to derive a regret bound that holds for any arbitrary online linear optimization problem.
	\begin{lemma}\cite[Lemma~E.2]{Duchi:SIAM2012}
	\label{lem:Duchi}
	Let $\f$ be a real-valued convex function on a closed domain which is a subset of $\RR^N$. Suppose $\f$ is $L$-Lipschitz with respect to $\|\cdot\|_2$, and let $\hat\f_\eta$ be the Gaussian smoothing of $\f$ with the scaling parameter $\eta$ and identity covariance. Then, $\{\hat\f_\eta\}$ is an $\eta$-smoothing of $\f$ with parameters $(L\sqrt{N}, L, \|\cdot\|_2)$.
	\end{lemma}
Consider an instance of online linear optimization with decision set $\cX$ and reward set $\cY$. The baseline potential function $\f$ is $\|\cX\|_2$-Lipschitz with respect to $\|\cdot\|_2$, where $\|\cX\|_2 = \sup_{x \in \cX} \|x\|_2$. From Lemma \ref{lem:Duchi} and Corollary \ref{lem:genericregret2}, it follows that
\[\Regret \leq \eta\sqrt{N} \|\cX\|_2 + \frac{\|\cX\|_2}{2}\sum_{t=1}^{T} \|\theta_t\|_2^2,\]
which is $O(N^{\frac{1}{4}}\|\cX\|_2\|\cY\|_2\sqrt{T})$ after tuning $\eta$.
This regret bound, however, often gives a suboptimal dependence on the dimension $N$. For example, it gives $O(N^{\frac{3}{4}}T^{\frac{1}{2})}$ regret bound for the experts setting where $\|\cX\|_2 = 1$ and $\|\cY\|_2 = \sqrt{N}$, and $O(N^\frac{1}{4}T^\frac{1}{2})$ regret bound  for the Euclidean balls setting where $\|\cX\|_2=\|\cY\|_2 = 1$.

\subsection{Online Convex Optimization} In online convex optimization, the learner receives a sequence of convex functions $f_t$ whose domain is $\cX$ and its subgradients are in the set $\cY$ \citep{DBLP:conf/icml/Zinkevich03}. After the learner plays $\w_t \in \cX$,  the reward function $f_t$ is revealed. The learner gains $f_t(\w_t)$ and observes  $\nabla f_t(\w_t)$,  a subgradient of $f_t$ at $\w_t$. 

A simple linearization argument shows that our regret bounds for online linear optimization generalize to online convex optimization.
Let $\w^*$ be the optimal fixed point in hindsight. The true regret is upper bounded by the linearized regret, as $f_t(\w^*) - f_t(\w_t) \leq \langle \w^* - \w_t, \nabla f_t(\w_t)\rangle$ for any subgradient $\nabla f_t(\cdot)$, and our analysis bounds the linearized regret. Unlike in the online linear optimization settings, however, the regret bound is valid only for the GBPA with smoothed potentials, which plays the \emph{expected action} of FTPL.

\section{Online Linear Optimization via Inf-conv Smoothing}
\label{sec:infconv}
\cite{journals/siamjo/BeckT12} proposed inf-conv smoothing, which is an infimal convolution with a strongly smooth function. In this section, we will show that FTRL is equivalent to the GBPA run with the inf-conv smoothing of the baseline potential function.

Let $(\mathcal{X}, \|\cdot\|)$ be a normed vector space, and $(\mathcal{X}^\star, \|\cdot\|_\star)$ be its dual. Let $\f: \mathcal{X}^\star \to \RR$ be a closed proper convex function, and let $\mathcal{S}$ be a $\beta$-strongly smooth function on $\cX^\star$ with respect to $\|\cdot\|_\star$. Then, the \emph{inf-conv smoothing} of $\f$ with $\mathcal{S}$ is defined as:
\begin{equation}
\label{eq:infconv}
		\f^{\ic}(\XX; \eta, \mathcal{S}) \overset{\text{def}}{=} \inf_{\Theta^* \in \mathcal{X}^\star}\Big\{\f(\Theta^*) + \eta \mathcal{S}\left(\frac{\XX - \XX^*}{\eta}\right)\Big\} = \max_{\w \in \mathcal{X}}\Big\{\langle \w, \XX \rangle - \f^\star(\w) - \eta \mathcal{S}^\star(\w)\Big\} .
\end{equation}
The first expression with infimum is precisely the infimal convolution of $\f(\cdot)$ and $\eta\mathcal{S}(\frac{\cdot}{\eta})$, and the second expression with supremum is an equivalent dual formulation. The inf-conv smoothing $\f^{\ic}(\XX; \eta, \mathcal{S})$ is finite, and it is an $\eta$-smoothing of $\f$ (Definition \ref{def:smoothable}) with smoothing parameters
\begin{equation}
\label{eq:infconv_const}
\left(\max_{\Theta \in \mathcal{X}^\star} \max_{\w \in \partial \f(\XX)} \mathcal{S}^\star(\w), \beta, \| \cdot \|	\right).
\end{equation}
where $\partial \f(\XX)$ is a set of subgradients of $\f$ at $\XX$.

\paragraph{Connection to FTRL}  Consider an online linear optimization problem with decision set $\mathcal{X} \subseteq \RR^N$.
Then, the dual space $\mathcal{X}^\star$ is simply $\RR^N$. 
Let $\cR$ be a $\beta$-strongly convex function on $\cX$ with respect to a norm $\|\cdot\|$. By the strong convexity-strong smoothness duality, $\mathcal{R}^\star$ is $\frac{1}{\beta}$-strongly smooth.
Consider the inf-conv smoothing of the baseline potential function $\f$ with $\cR^\star$, denoted $\f^{\ic}(\XX; \eta, \cR^\star)$. We will that the GBPA run with $\f^{\ic}(\XX; \eta, \cR^\star)$ is equivalent to FTRL with $\cR$ as the regularizer.

First, note that the baseline potential is the convex conjugate of the null regularizer, i.e., $\f^\star(\w) = 0$ for all $\w \in \mathcal{X}$. The dual formulation of inf-conv smoothing (Equation \ref{eq:infconv}) thus becomes
\[\f^{\ic}(\XX; \eta, \mathcal{S}) = \max_{\w \in \mathcal{X}}\Big\{\langle \w, \XX \rangle - \eta \cR(\w)\Big\},\]
which is identical to Equation \ref{eq:ftrl_potential} except that the above expression has an extra parameter $\eta$ that controls the degree of smoothing. To simplify the deviation parameter in Equation \ref{eq:infconv_const}, note that the subgradients of $\f$ always lie in $\mathcal{X}$ because of duality. Hence, the two supremum expressions collapse to one supremum: $\max_{w \in \mathcal{X}} \mathcal{S}^\star(\w)$. Plugging the smoothing parameters into Corollary \ref{lem:genericregret2} gives the well-known FTRL regret bound as in Theorem 2.11 or 2.21 of \cite{Sasha:2012:OLO}:
\[ \Regret \leq \eta \max_{w \in \mathcal{X}} \mathcal{S}^\star(\w) + \frac{\beta}{2\eta}\sum_{t=1}^{T}\|\theta_t\|^2. \]








\acks{CL and AT gratefully acknowledge the support of NSF under grant IIS-1319810. We thank the anonymous reviewers for their helpful suggestions. We would also like to thank Andre Wibisono for several very useful discussions and his help improving the manuscript. Finally we thank Elad Hazan for early support in developing the ideas herein.}

\newpage

\bibliography{colt2014}

\begin{thebibliography}{30}
\providecommand{\natexlab}[1]{#1}
\providecommand{\url}[1]{\texttt{#1}}
\expandafter\ifx\csname urlstyle\endcsname\relax
  \providecommand{\doi}[1]{doi: #1}\else
  \providecommand{\doi}{doi: \begingroup \urlstyle{rm}\Url}\fi

\bibitem[Abernethy et~al.(2008)Abernethy, Bartlett, Rakhlin, and
  Tewari]{DBLP:conf/colt/AbernethyBRT08}
Jacob Abernethy, Peter~L. Bartlett, Alexander Rakhlin, and Ambuj Tewari.
\newblock Optimal stragies and minimax lower bounds for online convex games.
\newblock In \emph{Proceedings of Conference on Learning Theory (COLT)}, 2008.

\bibitem[Abernethy et~al.(2013)Abernethy, Chen, and
  Vaughan]{DBLP:journals/teco/AbernethyCV13}
Jacob Abernethy, Yiling Chen, and Jennifer~Wortman Vaughan.
\newblock Efficient market making via convex optimization, and a connection to
  online learning.
\newblock \emph{ACM Transactions on Economics and Computation}, 1\penalty0
  (2):\penalty0 12, 2013.

\bibitem[Beck and Teboulle(2012)]{journals/siamjo/BeckT12}
Amir Beck and Marc Teboulle.
\newblock Smoothing and first order methods: A unified framework.
\newblock \emph{SIAM Journal on Optimization}, 22\penalty0 (2):\penalty0
  557--580, 2012.

\bibitem[Bertsekas(1973)]{bertsekas1973}
Dimitri~P. Bertsekas.
\newblock Stochastic optimization problems with nondifferentiable cost
  functionals.
\newblock \emph{Journal of Optimization Theory and Applications}, 12\penalty0
  (2):\penalty0 218--231, 1973.
\newblock ISSN 0022-3239.
\newblock \doi{10.1007/BF00934819}.

\bibitem[Bhatnagar(2007)]{Bhatnagar07adaptivenewton-based}
Shalabh Bhatnagar.
\newblock Adaptive newton-based multivariate smoothed functional algorithms for
  simulation optimization.
\newblock \emph{ACM Transactions on Modeling and Computer Simulation}, 2007.

\bibitem[Bishop(1995)]{Bishop:1995}
Chris~M. Bishop.
\newblock Training with noise is equivalent to tikhonov regularization.
\newblock \emph{Neural Computation}, 7\penalty0 (1):\penalty0 108--116, January
  1995.
\newblock ISSN 0899-7667.

\bibitem[Cesa-Bianchi and Lugosi(2006)]{DBLP:books/daglib/0016248}
Nicol{\`o} Cesa-Bianchi and G{\'a}bor Lugosi.
\newblock \emph{Prediction, learning, and games}.
\newblock Cambridge University Press, 2006.
\newblock ISBN 978-0-521-84108-5.

\bibitem[Devroye et~al.(2013)Devroye, Lugosi, and Neu]{Devroye:2013}
Luc Devroye, G{\'a}bor Lugosi, and Gergely Neu.
\newblock Prediction by random-walk perturbation.
\newblock In \emph{Proceedings of Conference on Learning Theory (COLT)}, 2013.

\bibitem[Duchi et~al.(2010)Duchi, Hazan, and Singer]{duchi_adagrad}
John Duchi, Elad Hazan, and Yoram Singer.
\newblock Adaptive subgradient methods for online learning and stochastic
  optimization.
\newblock In \emph{Proceedings of Conference on Learning Theory (COLT)}, 2010.

\bibitem[Duchi et~al.(2012)Duchi, Bartlett, and Wainwright]{Duchi:SIAM2012}
John Duchi, Peter~L. Bartlett, and Martin~J. Wainwright.
\newblock Randomized smoothing for stochastic optimization.
\newblock \emph{SIAM Journal on Optimization}, 22\penalty0 (2):\penalty0
  674--701, 2012.
\newblock \doi{10.1137/110831659}.

\bibitem[Freund and Schapire(1997)]{Freund1997119}
Yoav Freund and Robert~E. Schapire.
\newblock A decision-theoretic generalization of on-line learning and an
  application to boosting.
\newblock \emph{Journal of Computer and System Sciences}, 55\penalty0
  (1):\penalty0 119 -- 139, 1997.
\newblock ISSN 0022-0000.
\newblock \doi{http://dx.doi.org/10.1006/jcss.1997.1504}.

\bibitem[Glasserman(1991)]{glasserman1991gradient}
Paul Glasserman.
\newblock \emph{Gradient Estimation Via Perturbation Analysis}.
\newblock Kluwer international series in engineering and computer science:
  Discrete event dynamic systems. Springer, 1991.
\newblock ISBN 9780792390954.

\bibitem[Hannan(1957)]{hannan1957}
James Hannan.
\newblock Approximation to bayes risk in repeated play.
\newblock \emph{Contributions to the Theory of Games}, 3:\penalty0 97--139,
  1957.

\bibitem[Harsanyi(1973)]{harsanyi1973oddness}
John~C. Harsanyi.
\newblock Oddness of the number of equilibrium points: a new proof.
\newblock \emph{International Journal of Game Theory}, 2\penalty0 (1):\penalty0
  235--250, 1973.

\bibitem[Harvey(1965)]{Harvey1965}
James~R. Harvey.
\newblock Fractional moments of a quadratic form in noncentral normal random
  variables, April 1965.

\bibitem[Hinton et~al.(2012)Hinton, Srivastava, Krizhevsky, Sutskever, and
  Salakhutdinov]{Hinton:dropout}
Geoffrey~E. Hinton, Nitish Srivastava, Alex Krizhevsky, Ilya Sutskever, and
  Ruslan Salakhutdinov.
\newblock Improving neural networks by preventing co-adaptation of feature
  detectors.
\newblock \emph{ArXiv preprint}, arXiv:1207.0580, 2012.

\bibitem[Hofbauer and Sandholm(2002)]{sandholm}
Josef Hofbauer and William~H. Sandholm.
\newblock On the global convergence of stochastic fictitious play.
\newblock \emph{Econometrica}, 70\penalty0 (6):\penalty0 2265--2294, 2002.

\bibitem[Kalai and Vempala(2005)]{DBLP:journals/jcss/KalaiV05}
Adam~T. Kalai and Santosh Vempala.
\newblock Efficient algorithms for online decision problems.
\newblock \emph{Journal of Computer and System Sciences}, 71\penalty0
  (3):\penalty0 291--307, 2005.

\bibitem[McMahan(2011)]{DBLP:journals/jmlr/McMahan11a}
H.~Brendan McMahan.
\newblock Follow-the-regularized-leader and mirror descent: Equivalence
  theorems and l1 regularization.
\newblock In \emph{AISTATS}, pages 525--533, 2011.

\bibitem[Molchanov(2005)]{opac-b1102003}
Ilya~S. Molchanov.
\newblock \emph{Theory of random sets}.
\newblock Probability and its applications. Springer, New York, 2005.
\newblock ISBN 1-85233-892-X.

\bibitem[Nesterov(2011)]{Nesterov11}
Yurii Nesterov.
\newblock Random gradient-free minimization of convex functions.
\newblock \emph{ECORE Discussion Paper}, 2011.

\bibitem[Rakhlin et~al.(2012)Rakhlin, Shamir, and
  Sridharan]{DBLP:conf/nips/RakhlinSS12}
Alexander Rakhlin, Ohad Shamir, and Karthik Sridharan.
\newblock Relax and randomize : From value to algorithms.
\newblock In \emph{Proceedings of Neural Information Processing Systems
  (NIPS)}, 2012.

\bibitem[Rockafellar(1997)]{rockafellar}
R.T. Rockafellar.
\newblock \emph{Convex Analysis}.
\newblock Convex Analysis. Princeton University Press, 1997.
\newblock ISBN 9780691015866.

\bibitem[Shalev-Shwartz(2012)]{Sasha:2012:OLO}
Shai Shalev-Shwartz.
\newblock Online learning and online convex optimization.
\newblock \emph{Foundations and Trends in Machine Learning}, 4\penalty0
  (2):\penalty0 107--194, February 2012.
\newblock ISSN 1935-8237.

\bibitem[Srebro et~al.(2011)Srebro, Sridharan, and
  Tewari]{DBLP:conf/nips/SrebroST11}
Nati Srebro, Karthik Sridharan, and Ambuj Tewari.
\newblock On the universality of online mirror descent.
\newblock In \emph{Proceedings of Neural Information Processing Systems
  (NIPS)}, pages 2645--2653, 2011.

\bibitem[van Erven et~al.(2014)van Erven, Kotlowski, and
  Warmuth]{Warmuth:COLT2014}
Tim van Erven, Wojciech Kotlowski, and Manfred~K. Warmuth.
\newblock Follow the leader with dropout perturbations.
\newblock In \emph{Proceedings of Conference on Learning Theory (COLT)}, 2014.

\bibitem[Wager et~al.(2013)Wager, Wang, and Liang]{DBLP:conf/nips/WagerWL13}
Stefan Wager, Sida Wang, and Percy Liang.
\newblock Dropout training as adaptive regularization.
\newblock In \emph{Proceedings of Neural Information Processing Systems
  (NIPS)}, 2013.

\bibitem[Warmuth(2009)]{wm_gumbel}
Manfred Warmuth.
\newblock A perturbation that makes ``{F}ollow the leader'' equivalent to
  ``{R}andomized weighted majority''.
\newblock
  \url{http://classes.soe.ucsc.edu/cmps290c/Spring09/lect/10/wmkalai-rewrite.pdf},
  2009.

\bibitem[Youseﬁan et~al.(2010)Youseﬁan, Nedi\'c, and
  Shanbhag]{Yousefian:2010}
Farzad Youseﬁan, Angelia Nedi\'c, and Uday~V. Shanbhag.
\newblock Convex nondifferentiable stochastic optimization: A local randomized
  smoothing technique.
\newblock In \emph{Proceedings of American Control Conference (ACC), 2010},
  pages 4875--4880, June 2010.

\bibitem[Zinkevich(2003)]{DBLP:conf/icml/Zinkevich03}
Martin Zinkevich.
\newblock Online convex programming and generalized infinitesimal gradient
  ascent.
\newblock In \emph{International Conference on Machine Learning (ICML)}, 2003.

\end{thebibliography}

\newpage

\appendix

\section{Gradient-Based Prediction Algorithm}
\subsection{Proof of Lemma \ref{lem:genericregret1}}
\label{app:gbpa_regret}
\begin{proof}
We note that since $\f_0(0) = 0$,
\begin{align*}
\f_T(\XX_T) &= \sum_{t = 1}^{T} \f_{t}(\XX_{t}) - \f_{t - 1}(\XX_{t - 1}) \\
&= \sum_{t = 1}^{T} \Big( \big(\f_{t}(\XX_t) - \f_{t}(\XX_{t-1}) \big)+ \big(\f_{t}(\XX_{t-1}) - \f_{t - 1}(\XX_{t-1})\big) \Big) \\
\end{align*}
The first difference can be rewritten as:
\[\f_t(\XX_t) - \f_t(\XX_{t-1}) 
 = \langle  \nabla \f_t(\XX_{t-1}), \XX_{t}) \rangle + D_{\f_t}(\XX_{t}, \XX_{t - 1})\]
By combining the above two,
\begin{align*}
\Regret &= \f(\XX_T) - \sum_{t=1}^{T} \langle \nabla \f_t(\XX_{t-1}), \XX_t \rangle \notag\\
&= \f(\XX_T) - \f_T(\XX_T) + \sum_{t=1}^{T}  D_{\f_t}(\XX_t, \XX_{t-1}) + \f_t(\XX_{t-1}) - \f_{t-1}(\XX_{t-1}) \notag
\end{align*}
which completes the proof.
\end{proof} 

\section{FTPL-FTRL Duality}
\subsection{Proof of Theorem \ref{thm:ftpl-ftrl-bijecction}}
\label{app:ftpl-ftrl-bijecction}

\begin{proof}
Consider a one-dimensional online linear optimization prediction problem where the player chooses an action $\w_t$ from $\cX = [0, 1]$ and the adversary chooses a reward  $\theta_t$ from $\mathcal{Y} = [0, 1]$. This can be interpreted as a two-expert setting; the player's action $ w_t \in \cX $ is the probability of following the first expert and $\theta_t$ is the net excess reward of the first expert over the second. The baseline potential for this setting is $\f(\Theta) = \max_{w \in [0,1]} w\Theta$.

Let us consider an instance of FTPL with a continuous distribution $\mathcal{D}$ whose cumulative density function (cdf) is $F_\mathcal{D}$. Let $\tf$ be the smoothed potential function (Equation \ref{eq:ftpl_potential}) with distribution $\cD$. Its derivative is
\begin{equation}
\tf'(\Theta) = \EE[\argmax_{w \in K} w(\Theta + u)] = \PP[u > -\XX]
\end{equation}
because the maximizer is unique with probability 1.
Notice, crucially, that the derivative $\tilde{\f}'(\Theta)$ is exactly the expected solution of our FTPL instance. Moreover, by differentiating it again, we see that the second derivative of $\tf$ at $\Theta$ is exactly the pdf of $\mathcal{D}$ evaluated at $(-\XX)$.

We can now precisely define the mapping from FTPL to FTRL. Our goal is to find a convex regularization function $\cR$ such that 
$\PP(u > - \XX) = \arg\max_{\w \in {\cX}} \ (  \w \Theta  - \cR(\w))$. Since this is a one-dimensional convex optimization problem, we can differentiate for the solution. The characterization of $ \cR $ is:
\begin{equation}
\label{eq:ftpl_ftrl_bijection}
\cR(\w) - \cR(0) = - \int_0^\w F_{\mathcal{D}}^{-1}(1-z) \mathrm{d}z.
\end{equation}
Note that the cdf $ F_{\mathcal{D}}(\cdot) $ is indeed invertible since it is a strictly increasing function.

The inverse mapping is just as straightforward. Given a regularization function $\cR$ well-defined over $[0,1]$, we can always construct its Fenchel conjugate $\cR^\star(\Theta) = \sup_{w \in \cX} \langle w, \Theta \rangle - \cR(w)$. 
The derivative of $\cR^\star$ is an increasing convex function, whose infimum is 0 at $\Theta = -\infty$ and supremum is 1 at $\XX = +\infty$. Hence, $\cR^\star$  defines a cdf, and an easy calculation shows that this perturbation distribution exactly reproduces FTRL corresponding to $\cR$.
\end{proof}


\section{Gaussian smoothing}

\subsection{Proof of Equation \ref{eq:maxgauss}}
\label{app:maxgauss}
Let $X_1, \ldots, X_N$ be independent standard Gaussian random variables, and let $Z = \max_{i =1, \ldots, N } X_i$. For any real number $a$, we have
\[\exp(a\mathbb{E}[Z]) \leq \mathbb{E} \exp (a Z) = \mathbb{E} \max_{i = 1, \ldots, N}\exp (a X_i) \leq \sum_{i = 1}^N \mathbb{E} [\exp (a X_i)] = N \exp (a^2/2).\]
The first inequality is from the convexity of the exponential function, and the last equality is by the definition of the moment generating function of Gaussian random variables. Taking the natural logarithm of both sides and dividing by $a$ gives \[\EE[Z] \leq \frac{\log N}{a} + \frac{a}{2}.\] In particular, by choosing $a = \sqrt{2\log N}$, we have $\EE[Z] \leq \sqrt{2\log N}.$

\subsection{Proof of Lemma \ref{lem:changingf}}
\label{app:changingf}
\begin{proof}
By the subadditivity (triangle inequality) of $\Phi$,
\begin{align}
	\tilde{\f}(\Theta; \eta, \mathcal{N}(0, I)) - \tilde{\f}(\Theta; \eta', \mathcal{N}(0, I)) 
	&= \EE_{u \sim \mathcal{N}(0,I)}[\f(\XX + \eta u) - \f(\XX + \eta' u)]  \\
	&\leq \EE_{u \sim \mathcal{N}(0,I)}[\Phi((\eta - \eta')u)]
\end{align}
and the statement follows from the positive homogeneity of $\Phi$.
\end{proof}

\subsection{Proof that the origin is the worst case (Lemma \ref{thm:hessian_l2})}
\label{app:gaussian-l2worst}
\begin{proof}
Let $\Phi(\Theta) = \|\Theta\|_2$ and $\eta$ be a positive number. By  continuity of eigenvectors, it suffices to show that the maximum eigenvalue of the Hessian matrix of the Gaussian smoothed potential $\tilde{\f} (\XX; \eta, \mathcal{N}(0,I))$  is decreasing in $\|\XX\|$ for $\|\XX\|>0$.

By Lemma \ref{lem:smoothinggd}, the gradient can be written as follows:
\begin{equation}
\nabla \tilde{\Phi}(\XX; \eta, \mathcal{N}(0,I)) = \frac{1}{\eta}\EE_{u \sim \mathcal{N}(0,I)}[u\|\XX + \eta u\|]
\end{equation}
Let $u_i$ be the $i$-th coordinate of the vector $u$. Since the standard normal distribution is spherically symmetric, we can rotate the random variable $u$ such that its first coordinate $u_1$ is along the direction of $\XX$. After rotation, the gradient can be written as
$$\frac{1}{\eta} \EE_{u \sim \mathcal{N}(0, I)}\left[u \sqrt{(\|\XX\| + \eta u_1)^2 + \sum_{k=2}^{N} \eta^2 u_k^2} \right]$$
which is clearly independent of the coordinates of $\XX$. The pdf of standard Gaussian distribution has the same value at $(u_1, u_2,\ldots,u_n)$ and its sign-flipped pair $(u_1, -u_2,\ldots,-u_n)$. Hence, in expectation, the two vectors cancel out every coordinate but the first, which is along the direction of $\XX$. Therefore, there exists a function $\alpha$ such that $\EE_{u \sim \mathcal{N}(0,I)}[u\|\XX + \eta u\|] = \alpha(\|\XX\|)\XX.$

\newcommand{\uu}{a}
\newcommand{\bb}{b}
\newcommand{\BB}{b^2}
\newcommand{\rest}{\mathrm{rest}}
Now, we will show that $\alpha$ is decreasing in $\|\XX\|$. Due to symmetry, it suffices to consider $\XX = t e_1$ for $t \in \RR^+$, without loss of generality. For any $t > 0$,
\begin{align}
\alpha(t) &= \EE[u_1 \sqrt{(t + \eta u_1)^2 + u_{\rest}^2})] / t \notag\\
&= \EE_{u_{\rest}}[\EE_{u_1}[u_1\sqrt{(t + \eta u_1)^2 + \BB} | u_{\rest} = \bb ]] / t \notag\\
&= \EE_{u_{\rest}}[\EE_{\uu = \eta|u_1|}[a \Big(\sqrt{(t + \uu)^2 + \BB} - \sqrt{(t - \uu)^2 + B}\Big) | u_{\rest} = \bb ]] / t \notag
\end{align}

Let $g(t) = \Big(\sqrt{(t + \uu)^2 + B} - \sqrt{(t - \uu)^2 + B}\Big) / t$. Take the first derivative with respect to $t$, and we have:
\begin{align*}
g'(t) &= \frac{1}{t^2} \left( \sqrt{(t-\uu)^2 + b^2} - \frac{t(t-\uu)}{\sqrt{(t+\uu)^2 + b^2}} - \sqrt{(t+\uu)^2 + b^2} + \frac{t(t-\uu)}{\sqrt{(t+\uu)^2 + b^2}} \right)\\
&= \frac{1}{t^2}\left(\frac{a^2 + b^2 - a t}{\sqrt{(t - a)^2 + b^2}}- \frac{a^2 + b^2 + a t}{\sqrt{(t + a)^2 + b^2}} \right)
\end{align*}
$$\Big((a^2 + b^2) - a t \Big)^2 
\Big( (t + a)^2 + b^2 \Big) 
- 
\Big((a^2 + b^2) + at\Big)^2 
\Big( (t - a)^2 + b^2 \Big) 
= -4ab^2t^3 < 0$$
because $t, \eta, u', B$ are all positive. So, $g(t) < 0$, which proves that $\alpha$ is decreasing in $\XX$.

The final setp is to write the gradient as $\nabla (\tilde{\f}; \eta, \mathcal{N}(0, I))(\XX) = \alpha(\|\XX\|)\XX$ and differentiate it:
\[\nabla^2 f_\eta(\XX) = \frac{\alpha'(\|\XX\|)}{\|\XX\|}\XX\XX^T + \alpha(\|\XX\|)I\]
The Hessian has two distinct eigenvalues $\alpha(\|\XX\|)$ and $\alpha(\|\XX\|) + \alpha'(\|\XX\|) \|\XX\|$, which correspond to the eigenspace orthogonal to $\XX$ and parallel to $\XX$, respectively. Since $\alpha'$ is negative, $\alpha$ is always the maximum eigenvalue and it decreases in $\|\XX\|$. 
\end{proof}

%

\end{document}